\documentclass{article}

\usepackage{aistatsFiles/aistats2014} 

\usepackage{times}
\usepackage{graphicx} 
\usepackage{subfigure,wrapfig}
\usepackage{algorithm}
\usepackage{algorithmic}

\usepackage{enumitem}
\usepackage{amsmath,amssymb,amsfonts,amsthm,xfrac}
\usepackage{capt-of}

\DeclareMathOperator*{\argmin}{arg\,min}

\newtheorem{theorem}{Theorem}[section]
\newtheorem{lemma}[theorem]{Lemma}
\theoremstyle{definition}

\theoremstyle{remark}

\newcommand{\norm}[1]{\lVert#1\rVert}
\newcommand{\Norm}[1]{\left\lVert#1\right\rVert}

\newcommand{\EE}[1]{\mathbb{E}\left[ #1 \right]}
\newcommand{\Idot}[2]{\left\langle #1 , #2 \right\rangle}
\newcommand{\idot}[2]{\langle #1 , #2 \rangle}

\newcommand{\iid}{\overset{iid}{\sim }}

\newcommand{\R}{\mathbb{R}}
\newcommand{\Z}{\mathbb{Z}}

\newcommand{\E}{\mathbb{E}}

\newcommand{\bZ}{{\bf{Z}}}

\newcommand{\calF}{\mathcal{F}}

\newcommand{\calR}{\mathcal{R}}
\newcommand{\calO}{\mathcal{O}}
\newcommand{\calI}{\mathcal{I}}

\newcommand{\calD}{\mathcal{D}}

\newcommand{\calG}{\mathcal{G}}

\newcommand{\calN}{\mathcal{N}}

\newcommand{\vY}{\vec{A}}
\newcommand{\bY}{{\bf{A}}}

\newcommand{\subI}{{\scriptscriptstyle \calI}}
\newcommand{\subO}{{\scriptscriptstyle \calO}}

\newcommand{\tp}{\tilde{p}}

\newcommand{\tq}{\tilde{q}}

\newcommand{\hq}{\hat{q}}
\newcommand{\hQ}{Q}

\newcommand{\hP}{P}

\newcommand{\hpsi}{\hat{\psi}}
\newcommand{\hPsi}{\hat{\Psi}}

\newcommand{\bA}{A}

\newcommand{\avec}{\vec{a}}

\newcommand{\ud}{\mathrm{d}}

\newcommand{\mxphi}{\varphi_{\mathrm{max}}}

\newcommand{\vp}{P}
\newcommand{\vq}{Q}

\newcommand{\bsetI}{U}
\newcommand{\bsetO}{V}

\newcommand{\where}{\mathrm{where }}

\DeclareMathOperator*{\Var}{Var}
\DeclareMathOperator*{\Unif}{Unif}

\newcommand{\VVar}[1]{\Var\left[ #1 \right]}

\begin{document}

\twocolumn[ 

\aistatstitle{Fast Function to Function Regression}


\aistatsauthor{Junier Oliva \quad Willie Neiswanger 
                \quad Barnabas Poczos \quad Eric Xing \quad Jeff Schneider}
\aistatsaddress{ Machine Learning Department\\Carnegie Mellon University  }
]


\begin{abstract} 
We analyze the problem of regression when both input covariates and output
responses are functions from a nonparametric function class. Function to
function regression (FFR) covers a large range of interesting applications
including time-series prediction problems, and also more general
tasks
like studying a mapping between two separate types of distributions.
However, previous nonparametric estimators for FFR type problems scale badly computationally with the number of input/output pairs in a data-set.
Given the complexity of a mapping between general functions it may be
necessary to consider large data-sets in order to achieve a low estimation risk.
To address this issue, we develop a novel scalable nonparametric estimator, the Triple-Basis Estimator (3BE), which is capable of operating over datasets with many instances.
To the best of our knowledge, the 3BE is the first nonparametric FFR estimator that can scale to massive datasets. 
We analyze the 3BE's risk and derive an upperbound rate. Furthermore, we show an improvement of several orders of magnitude in terms of prediction speed and a reduction in error over previous estimators in various real-world data-sets.
\end{abstract}


\section{Introduction}
\label{introduction}
Modern data-sets are not only growing in
quantity of instances but the instances themselves are growing in complexity
and dimensionality. The goal of this paper is to perform regression with
data-sets that are massive not only in terms of the number of instances but also in terms of
the complexity of instances; specifically we consider functional data. We study function to function regression
(FFR) where one aims to learn a mapping $f$ that takes in a general input functional
covariate $p:\R^l\mapsto\R$ and outputs a functional response
$q=f(p):\R^k\mapsto\R$. In general, functions are infinite
dimensional objects; hence, the problem of FFR is not immediately solvable by
traditional regression methods on finite vectors. Furthermore, unlike with typical regression
problems, neither the covariate nor the response will be directly observed
(since it is infeasible to directly observe functions). Previous nonparametric estimators for FFR do not scale computationally to large data-sets. However, 
large data-sets are often needed to achieve a low risk; to mitigate this issue we introduce the Triple-Basis Estimator (3BE).

The FFR framework is quite general and includes many interesting problems. For instance, one may consider input/output functions that are probability distribution functions (pdfs). An
example of a financial domain related FFR problem with density functions is learning the
mapping that takes in the pdf of stock prices in a specific industry and
outputs the pdf of stock prices in another industry. 
Additionally, in cosmology one may be interested in regressing a mapping that takes in the pdf of simulated particles from a computationally inexpensive but inaccurate simulation and outputs the corresponding pdf of particles from a computationally expensive but accurate simulation. In essence, one would be enhancing the inaccurate simulation using previously seen data from accurate simulations. There are also many non-distributional FFR problems. For example, one may view foreground/background segmentation as a FFR problem that maps an image's $p$ function to a segmentation's $q$ function, where $p(x,y)$ is a function that takes in a pixel's $(x,y)$ position and outputs the corresponding pixel's intensity, and $q(x,y)$ is function that takes in a pixel's position and outputs $1$ if the pixel is in the foreground and $0$ otherwise.

\begin{figure}[t!]
        \centering
        \subfigure[Forward Prediction]{{\label{fig:time-series-forward}\includegraphics[width=.2\textwidth]{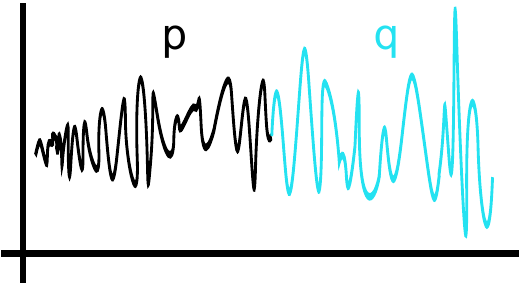}}}
        \hspace{1mm}
        \subfigure[Co-occurring Series]{{\hspace{.1cm}\label{fig:time-series-co-occur}\includegraphics[width=.2\textwidth]{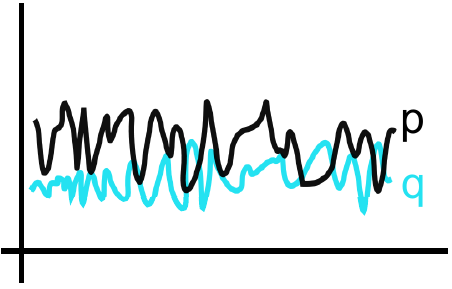}}}
        \caption{
            (a) One may consider trying to predict a later portion of a
            time-series when given the prior portion of a time-series as a FFR
            problem. (b) One may try to predict one co-occurring time-series
            when given another. 
        }\label{fig:time-series}
\end{figure}

Moreover, several time-series tasks may be posed in the FFR framework (see Figure
\ref{fig:time-series}). Suppose, for example, that one is interested in
predicting the next unit interval of a time-series given the previous unit interval; then, one may frame this as a FFR problem by letting
input functions $p : [0,1] \mapsto \R$ be the function representing the time-series
during the first unit interval and output functions $q:[0,1]\mapsto\R$ be the
function representing the time-series during the next unit interval (Figure
\ref{fig:time-series-forward}). A related problem is that of predicting
co-occurring functions (Figure \ref{fig:time-series-co-occur}). 
An interesting application of predicting
co-occurring functions is with motion capture data, where one may be interested
in predicting the movement of joints that are occluded given the movement of
observed joints. 

As stated previously, the problem of FFR boils down to the study of a mapping
between infinite dimensional objects. Thus, the regression task would benefit
greatly from learning on data-sets with a large number of input/output pairs.
However, many nonparametric estimators for regression problems \emph{do not scale 
well } in the number of instances in a data-set. Thus, if the number of instances is in
the many thousands, millions, or even more, then it will be infeasible to use such
an estimator. This leads to a paradox: one wants many instances in a
data-set in order to effectively learn the FFR mapping, but one also wants a
low number of instances in order to avoid a high computational cost. We resolve this issue through the 3BE, which we will show can perform FFR in a scalable manner.

The data-sets we consider are as follows. Since general functions are infinite dimensional we cannot work over a data-set
$\bar{\calD} = \{(p_i,q_i)\}_{i=1}^N$ where $q_i=f(p_i)$. Instead we shall work
with a data-set of instances that are (inexact) observation pairs from input/output functions
$\calD=\{(\vp_i,\vq_i)\}_{i=1}^N$ where $\vp_i$, and $\vq_i$ are some form of
empirical observations from $p_i$ and $q_i$ (see Figure \ref{fig:gmodel}). For
example, one may consider the functional observations to be a set of $n$ noisy
function evaluations at uniformly distributed points, or a sample of $n$ points drawn from $p$ and $q$ respectively
(when $p,q$ are distributions). Using $\calD$ we will make an estimate of
$\bar{\calD}$ as $\tilde{\calD}=\{(\tp_i,\tq_i)\}_{i=1}^N$ where $\tp_i,\tq_i$ are
functional estimates created using $\vp_i,\vq_i$ respectively. The task then is
to estimate $q_0=f(p_0)$ as $\hat{q}_0=\hat{f}(\tp_0)$ when given a functional observation, $\vp_0$, of an
unseen function $p_0$. 

\begin{figure}[t]
  \centering
    \includegraphics[width=0.45\textwidth]{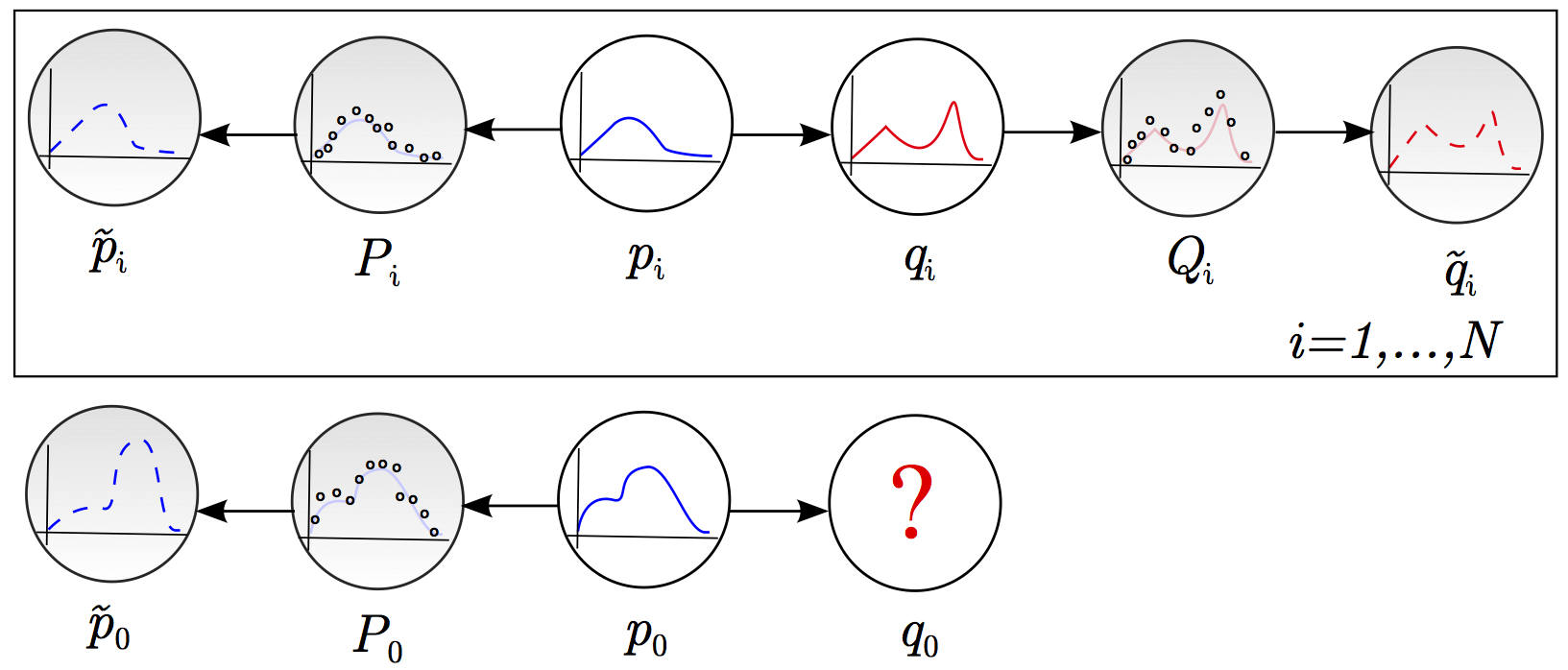}
    \caption{
    We observe a data-set of input/output
    functional observation pairs $\{(\vp_i,\vq_i) \}_{i=1}^N$, where $\vp_i$,
    $\vq_i$ are some inexact observations of functions $p_i$ and $q_i$ such as
    a set of noisy function evaluations at uniformly distributed points.
    $\vp_i$, $\vq_i$ then are used to make function approximations
    $\tp_i,\tq_i$, which in turn are used to predict the response $q_0$ for a
    unseen query input function $p_0$. 
    }
    \label{fig:gmodel}
\end{figure}

Our approach will be as follows. First, we convert the infinite dimensional task of estimating the output function $q_0$ into a finite dimensional problem by projecting $q_0$ into a finite number of basis functions (focusing on the crucial characteristics of $q_0$, roughly speaking). Then, to estimate the projections onto the basis functions we embed the input functions into a nonlinear space where linear operations are approximately evaluations of a nonlinear mapping $f$ in a broad function class. Finally, $f$ is estimated by empirically minimizing the risk of a linear operation in the nonlinear embedding of input functions for predicting the basis projections of output functions in a dataset.

\paragraph*{Our Contribution}
We develop the Triple-Basis Estimator (3BE), a novel nonparametric estimator for FFR that
scales to large data-sets. The 3BE is the first estimator of its kind, allowing one to regress functional responses given functional covariates in massive data-sets. 
Furthermore, we analyze the $L_2$ risk of the 3BE
under nonparametric assumptions. Lastly, we show an improvement of several orders of magnitude over existing estimators in terms of prediction time as well as a reduction in error in various real-world data-sets.


\section{Related Work}
\label{relatedwork}
A previous nonparametric FFR estimator was proposed in \cite{kadri2010nonlinear}. \cite{kadri2010nonlinear} attempts to perform FFR on a functional
RKHS. 
That is, if we consider $\calF$ as a functional Hilbert space, where
$f\in\calF$ is such that $f:\calG_x\mapsto\calG_y$, then $f$ is estimated by 
$
f^* = \argmin_{\hat{f}} \sum_{i=1}^N \norm{q_i-\hat{f}(p_i)}^2_{\calG_y} + \lambda \norm{f}^2_{\calF}.
$
However, when each function is observed though $n$ noisy function evaluations
this estimator will require the inversion of a $Nn\times Nn$ matrix, which will
be computationally infeasible for data-sets of even a modest size. 

In addition, \cite{olivadistribution} provides an estimator for doing
FFR for the special case where both input and output functions
are probability distribution functions. The estimator, henceforth referred to as the linear smoother estimator
(LSE), works as follows when given a training data-sets of $\calD=\{(P_i,Q_i)\}_{i=1}^N$ of empirical functional observations and
$\tilde{\calD}=\{(\tp_i,\tq_i)\}_{i=1}^N$ of function estimates and a function
estimate $\tp_0$ of a new query input function: 
\begin{align}
\hat{f}(\tp_0) &= \sum_{i=1}^N
W(\tp_i,\tp_0) \tq_i\ \where\  \\
W(\tp_i,\tp_0)
&= \begin{cases} \frac{K(D(\tp_i,\tp_0))}{\sum\limits_{j=1}^N K(D(\tp_j,\tp_0))} &\mbox{if } \sum_j K(D(\tp_j,\tp_0))>0 \\
0 & \mbox{otherwise }. \end{cases} \label{eq:LSE}
\end{align}
Here $K:\R\rightarrow [0,\infty)$ is taken to be a symmetric kernel with bounded support, and $D$ is
some metric over functions. 
However, while such an estimator is useful for smaller FFR problems, it may not be used in larger data-sets.
Clearly, the LSE must perform a kernel evaluation with all input distributions in one's data-set to produce a prediction, leading to a total computational cost of $\Omega(Nn)$ when considering the cost of computing metrics $D(\tp_j,\tp_0)$ when $|P_i| \asymp |Q_i| \asymp n$. This implies, for example, that obtaining estimates for each training instance scales as $\Omega(n N^2)$, which will be prohibitive for big data-sets. 

Previous work for nonparametric estimation in large data-sets with functional inputs includes
\cite{oliva2013fast}. There an estimator was proposed for scalable
learning of a distribution input covariate to real-value output response
regression problem. We note however that it is not immediately clear how to achieve a scalable
estimator for regression functional responses with functional covariates, nor how to analyze such
an estimator's risk since general functional responses
are infinite dimensional.

We note
further that work has been done in linear models for FFR (e.g.
\cite{ramsay2006functional, oliva2013fusso}). However, such models work over a
strong assumption on the linearity of the mapping $f$, and will not be able to
capture non-linear mappings. Moreover, FFR is a specific case of
general functional analysis
\cite{ramsay2006functional,ferraty2006nonparametric, ramsay2002applied}.


\section{Model}
\label{model}

We expound upon our model of input/output functions and the mapping between them. Later, we introduce the 3BE and its risk for the case when one has a data-set of pairs of input/output functional
observations that are a set of noisy function evaluations at uniformly distributed
points. However, the following is generalizable for the case where one observes function evaluations at a fixed grid of points or function observations of samples from distributions. In short, we assume smooth input/output functions that are well approximated by a finite number of basis functions. Further, we consider a nonparametric mapping between them, where the projection of the output function onto each basis function may be written as an infinite linear combination of RBF kernel evaluations between the input function and unknown functions (see below).

We take our data-set to be input/output empirical function observation pairs:
\begin{align}
\calD &= \{(\vp_i,\vq_i)\}_{i=1}^N \ \where\ \\
\vp_i &= \left\{ p_i(u_{ij})+\epsilon_{ij} \right\}_{j=1}^{n_i},\ \vq_i = \left\{ q_i(v_{ij})+\xi_{ij} \right\}_{j=1}^{m_i}, \label{eq:func_obs}
\end{align}
with sample points $u_{ij}\iid \Unif([0,1]^l)$, $v_{ij}\iid \Unif([0,1]^k)$,
and noise $\epsilon_{ij}\iid\Xi_\epsilon$, $\xi_{ij}\iid\Xi_\xi$. With error distributions $\Xi_\xi,\Xi_\epsilon$, s.t.
$\EE{\epsilon_{ij}}=\EE{\xi_{ij}}=0$,
$\VVar{\epsilon_{ij}},\VVar{\xi_{ij}}\leq\varsigma < \infty$. Furthermore,
$p_i\in \calI$, $p_i:[0,1]^l\mapsto \R$, $q_i\in \calO$, $q_i:[0,1]^k\mapsto
\R$, $q_j = f(p_j)$\footnote[1]{Similarly, one may consider a model $q_i(x) =
[f(p_i)](x) + \zeta w_i(x)$, where $w_i$ is a standard Wiener process. This
however will be akin to adding variance to our noisy function evaluations,
hence we omit $w_i$ for simplicity.}, and $p_i\iid\Phi$ where $\calI$ and
$\calO$ are some class of input/output functions and $\Phi$ is some measure
over $\calI$. Furthermore, we shall assume that $n_i\asymp n$ and $m_i \asymp
m$. We shall use $\calD$ to make estimates of the true input/output functions
$\tilde{D}=\{(\tp_i,\tq_i)\}_{i=1}^N$, which will then be used to estimate the
output function $q_0$ corresponding to an unseen input function $p_0$. 

\subsection{Basis Functions and Projections}
Let $\{\varphi_i\}_{i\in\Z}$ be an orthonormal basis for $L_2([0,1])$. Then,
the tensor product of $\{\varphi_i\}_{i\in\Z}$ serves as an orthonormal basis
for $L_2([0,1]^d)$; that is, the following is an orthnormal basis for $L_2([0,1]^d)$
\begin{gather*}
\{\varphi_\alpha\}_{\alpha\in\Z^d} \quad \mathrm{where} \quad \varphi_\alpha(x) = \prod_{i=1}^d \varphi_{\alpha_i}(x_i),\ x\in [0,1]^d.
\end{gather*}
So we have that $\forall \alpha,\rho\in \Z^d,\ \langle\varphi_\alpha,\varphi_\rho \rangle= I_{\{\alpha=\rho\}}$.
Let $h\in L_2([0,1]^d)$, then 
\begin{align}
    h(x)&=\sum_{\alpha\in\Z^d} a_\alpha(h)\varphi_\alpha(x)  \quad \where \\
\quad a_\alpha(h) &= \langle\varphi_\alpha, h\rangle = \int_{[0,1]^d} \varphi_\alpha(z)h(z)\ud z\ \in \R.
\end{align}

As previously mentioned, a data-set of estimated input/output function pairs,
$\tilde{\calD}=\{(\tp_i,\tq_i)\}_{i=1}^N$, will be constructed from the
data-set of input/output function evaluation sets
$\calD=\{(\vp_i,\vq_i)\}_{i=1}^N$. Suppose function
$h$ has a corresponding set of evaluations $H =
\{y_j = h(u_j)+ \epsilon_j\}_{j=1}^r$ where $u_j\iid \Unif([0,1]^d)$ and
$\EE{\epsilon_j}=0$, $\EE{\epsilon_j^2}<\infty$. Then, $\tilde{h}$, the estimate of $h$, will
be as follows:
\begin{align}
\tilde{h}(x) &= \sum_{\alpha \in M }a_\alpha(H)\varphi_\alpha(x) \quad \where \\
a_\alpha(H) &= \frac{1}{r}\sum_{j=1}^{r} y_j\varphi_\alpha(u_j) \label{eq:coef-est},
\end{align}
and $M$ is a finite set of indices for basis functions.

\subsubsection{Cross-validation}
\label{sec:pccv}
In practice, one would choose indices $M$ in \eqref{eq:coef-est} through cross-validation. The number of projection coeffients one chooses will depend on the smoothness of the function $h$ aswell as the number of points in $H$. Typically, a larger $|i|$ will correspond to a higher frequency 1-dimensional basis function $\varphi_i$; thus, a natural way of selecting $M$ is to consider sets 
\begin{align}
M_t=\{\alpha \in \Z^d : \norm{\alpha}_2 \leq t\} \label{eq:cv_Mt}
\end{align} with $t\in (0,\infty)$. One would then choose the value of $t$ (setting $M=M_t$) that minimizes a loss, such as the mean squared error between $\tilde{h}(u_i)$ and $y_i$. We shall see below that considering $M_t$ in this manner corresponds to a smoothness assumption on the class of input/output functions.

\subsection{Function to Function Mapping}
Let $p\sim\Phi$, $q=f(p)$, as in \eqref{eq:sob-ellp-inp}, we have that
\begin{align}
q(x) = [f(p)](x) 
&= \sum_{\alpha\in\Z^k} a_\alpha(f(p))\varphi_\alpha(x) 
\\&= \sum_{\alpha\in\Z^k} f_\alpha(p)\varphi_\alpha(x)
\end{align}
where $f_\alpha(p) = a_\alpha(f(p))$. Hence we may think of
$f:\calI\mapsto\calO$ as consisting of countably many functions $\{f_\alpha\ |\
f_\alpha:\calI\mapsto\R,\ \alpha\in\Z^k\}$, where each $f_\alpha$ is responsible for the mapping of $p$ to the projection of $q$ on to $\varphi_\alpha$. We take $f_\alpha$ functions 
to be a nonparametric linear smoother on a possibly infinite set of functions weighted by a kernel:
\begin{gather}
f_\alpha(p) = \sum_{i=1}^{\infty} \theta_{\alpha i} K_{\sigma}(g_{\alpha i}, p)\ \where \label{eq:q_form} \\
\theta_{\alpha i}\in\R,\ g_{\alpha i} \in \calI
\end{gather}
We shall consider the following class of functions:
\begin{align}
\calF_\sigma &= \{f\ :\ \forall \alpha\in\Z^k\ \norm{\theta_{\alpha}}_1\leq B_{\alpha}, f_{\alpha} \text{ as in } \eqref{eq:q_form}  \}
\label{eq:func_class_set}.
\end{align}


\section{Triple-Basis Estimator}
\label{triplebasis}
If the tail-frequency behavior of output functions are controlled, then we may
effectively estimate output functions using a finite number of projection
coefficients; thus, we only need to estimate a finite number of the $f_\alpha$
functions. The 3BE consists of two orthonormal bases for estimating input and output functions respectively and a random basis to estimate the mapping between them. To efficiently estimate the $f_\alpha$ functions, we shall
use random basis functions from Random Kitchen Sinks (RKS)
\cite{rahimi2007random}. We shall show that to approximate $f_\alpha$, we need only estimate a linear mapping in the random RKS features. \cite{rahimi2007random} shows that if one has a
shift-invariant kernel $K$ (in particular we consider the RBF kernel
$K(x)=\exp(-x^2/2)$), then for fixed $\omega_i \stackrel{iid}{\sim} \calN(0,\sigma^{-2}I_d)$, $b_i \stackrel{iid}{\sim} \Unif([0,2\pi])$, we have that for each $x,y \in \R^d$:
 \begin{align}
    &K(\Norm{x-y}_2/\sigma) \approx z(x)^Tz(y),\ \where \\
    &z(x) \equiv
 \sqrt{\tfrac{2}{D}}\left[\cos(\omega_1^Tx+b_1) \cdots \cos(\omega_D^Tx+b_D)\right]^T,
 \label{eq:rks_feats}
\end{align}
and $D$ is the number of random basis functions (see \cite{rahimi2007random} for approximation quality) . Let $\bsetI$ and $\bsetO$ be a set of indices for basis functions to project input and output functions respectively:
\begin{gather}
\bsetI = \{\alpha_1,\ldots,\alpha_s\}, 
\bsetO = \{\beta_1,\ldots,\beta_r\}. \label{eq:Mt}
\end{gather}
In practice one would choose $\bsetI$ and $\bsetO$ through cross-validation (see $\S$\ref{sec:pccv}).
First note that:
\begin{align}
\idot{\tp_i}{\tp_j}  &=  \Idot{\sum_{\alpha\in \bsetI} a_{\alpha}(\hP_i) \varphi_{\alpha}}{\sum_{\alpha\in \bsetI} a_{\alpha}(\hP_j) \varphi_{\alpha}}  \\
 &=  \sum_{\alpha\in \bsetI} \sum_{\beta\in \bsetI}a_{\alpha}(\hP_i) a_{\beta}(\hP_j)\Idot{\varphi_{\alpha}}{\varphi_{\beta}} \\
 =&  \sum_{\alpha\in \bsetI} a_{\alpha}(\hP_i) a_{\alpha}(\hP_j)
 = \Idot{\avec_{\bsetI}(\hP_i)}{\avec_{\bsetI}(\hP_j)},
\end{align}
where $\avec_{\bsetI}(\hP_i) = (a_{\alpha_1}(\hP_i),\ldots,a_{\alpha_s}(\hP_i))^T$. Thus,
$
\Norm{\tp_i-\tp_j}_2 = \Norm{\avec_{\bsetI}(\hP_i)-\avec_{\bsetI}(\hP_j)}_2,
$
where the norm on the LHS is the $L_2$ norm and the $\ell_2$ on the RHS. 

Consider a fixed $\sigma$, and let $\omega_i \stackrel{iid}{\sim} \calN(0,\sigma^{-2}I_s)$, $b_i \stackrel{iid}{\sim} \Unif[0,2\pi]$, be fixed. Then
\begin{align}
f_\alpha(p_0)
&=\sum_{i=1}^{\infty} \theta_{\alpha i} K_\sigma(\norm{g_{\alpha i}-p_0}_2)\\ &\approx \sum_{i=1}^{\infty} \theta_{\alpha i} K_\sigma(\norm{\avec_{\bsetI}(g_{\alpha i})-\avec_{\bsetI}(\hP_0)}_2) \\
&\approx \sum_{i=1}^\infty \theta_{\alpha i}z(\avec_{\bsetI}(g_{\alpha i}))^Tz(\avec_{\bsetI}(\hP_0))\\
&= 
 \hspace{1mm} \psi_\alpha^T z(\avec_{\bsetI}(\hP_0)) \label{eq:lin_est_approx}
\end{align}
where $\psi_\alpha = \sum_{i=1}^\infty \theta_{\alpha i}z(\avec_{\bsetI}(g_{\alpha i})) \in \R^s$. Hence, by
\eqref{eq:lin_est_approx} $f_\alpha$ is approximately linear in $z(\avec_{\bsetI}(\cdot))$; so, we consider linear estimators in the non-linear space induced by $z(\avec_{\bsetI}(\cdot))$. In particular, we take the OLS estimator using the data-set $\{(z(\avec_{\bsetI}(\hP_i)),a_\alpha(\hQ_i))\}_{i=1}^N$, and for each $f_\alpha$ we estimate :
\begin{align}
    \hat{f}_\alpha(\vp_0) &\equiv \hpsi^T_\alpha z(\avec_{\bsetI}(\hP_0))\quad\where\\
    \hpsi_\alpha &\equiv \argmin_\beta \norm{\vY_\alpha-\bZ\beta}_2^2 =                (\bZ^T\bZ)^{-1}\bZ^T\vY_\alpha \label{eq:OLSest}
\end{align}
for $\vY_\alpha=(a_\alpha(\hQ_1),\ldots,a_\alpha(\hQ_N))^T$, and $\bZ$ the $N\times D$ matrix $\bZ=[z(\avec_\bsetI(\hP_1))\cdots z(\avec_\bsetI(\hP_N)) ]^T$. 
Suppose that the indices of basis functions we project output function onto is $\bsetO$ (as in \eqref{eq:Mt}), then the set of functions we estimate is $\{\hat{f}_\alpha\ :\
\alpha\in \bsetO\}$. Let $\hat{f}_{1:r}(\vp_0) =
(\hat{f}_{\alpha_1}(\vp_0),\ldots,\hat{f}_{\alpha_r}(\vp_0))^T$, $\bY_{1:r}=[\vY_{\alpha_1},\ldots,\vY_{\alpha_r}]\in\R^{N\times r}$:
\begin{align}
\hat{f}_{1:r}(\vp_0) &= \hPsi^T z(\avec_{\bsetI}(\vp_0))\quad\where\\
\hPsi &= (\bZ^T\bZ)^{-1}\bZ^T\bY_{1:r} \label{eq:approx-pc}.
\end{align}

\subsection{Evaluation Computational Complexity}
We see that after computing $\hPsi$, evaluating the estimated projection
coefficients for a new function $p_0$ amounts to performing a matrix
multiplication of a $r\times D$ matrix with a $D \times 1$ vector. Including
the time required for computing $z(\avec_U(\hP_0))$, the computation required
for the evaluation, \eqref{eq:approx-pc}, is: 1) the time for evaluating the
projection coefficients $\avec_U(\hP_0)$, $O(sn)$; 2) the time to compute the
RKS features $z(\cdot)$, $O(Ds)$; 3) the time to compute the matrix
multiplication, $\hPsi^T z(\avec_U(\hP_0))$, $O(rD)$. Hence, the total time is
$O(rD+Ds+sn)$. 

We'll see that we may choose $D=O(n\log(n))$, $s=O(n)$, and $r=O(m)$. If we assume further
that $m\asymp n$, the total runtime for evaluating $\hat{f}(\tp_0)$ is
$O(n^2\log(n))$. Since we are considering data-sets where the number of
instances $N$ far outnumbers the number of points per sample set $n$,
$O(n^2\log(n))$ is a \emph{substantial improvement} over $\Omega(Nn)$ for the LSE; indeed, the 
LSE requires a metric evaluation with every training-set input function \eqref{eq:LSE} where the 3BE does not. Furthermore, the
space complexity is much improved for the 3BE since we only
need to store the $O(n^2\log(n))$ matrix $\Psi$ and the $O(n^2\log(n))$ total
space for the RKS basis functions $\{(\omega_i,b_i)\}$. Contrast this with the
space required for the LSE, $\Omega(Nn)$, which is much larger
for our case of $n\ll N$. Lastly, note that to evaluate $\hat{q}_0(x) =
[\hat{f}(\hP_0)](x)$ once one has computed $\hat{f}_{1:r}(\hP_0)$, one only needs to
compute $\hat{q}_0(x) = \langle \hat{f}_{1:r}(\hP_0),\vec{\varphi}_{1:r}(x) \rangle$ where
$\vec{\varphi}_{1:r}(x) = (\varphi_{\alpha_1}(x) ,\ldots,
\varphi_{\alpha_r}(x))$.

\paragraph{Triple-Basis Estimator}
We note that a straightforward extension to the 3BE is to
use a ridge regression estimate on features $z(\avec_t(\cdot))$ rather than a
OLS estimate. That is, for $\lambda\geq 0$ let
\begin{align}
\hpsi_{\alpha\lambda} &\equiv \argmin_\beta \norm{\vY_\alpha-\bZ\beta}_2^2 + \lambda \norm{\beta}_2^2 \\
&= (\bZ^T\bZ+\lambda I)^{-1}\bZ^T\vY_\alpha \label{eq:ridgeest}.
\end{align}
The Ridge-3BE is still evaluated via a matrix multiplication,
and our complexity analysis holds. 

\subsection{Algorithm}
We summarize the basic steps for training the 3BE in practice given a data-set of empirical functional observations $\calD=\{(\vp_i,\vq_i)\}_{i=1}^N$, parameters $\sigma$ and $D$ (which may be cross-validated), and an orthonormal basis $\{\varphi_i\}_{i\in\Z}$ for $L_2([0,1])$.
\begin{enumerate}
  \item Determine the sets of basis functions $U$ and $V$ \eqref{eq:Mt} for approximating $p$, and $q$ respectively. For each $j$ in a subset $J \subseteq \{1,\ldots,N\}$\footnote{Empirically it has been observed that $\bar{t}$ and $\bar{c}$ perform well even when $|J|$ is much smaller than $N$} one can select a set $M_{t_j}$ \eqref{eq:cv_Mt} to estimate $p_j$ by cross-validating a loss as described in $\S$ \ref{sec:pccv}. One may then set $U = M_{\bar{t}}$ where $\bar{t} = \frac{1}{|J|}\sum_{j\in J} t_j$. Similarly, one may set $V = M_{\bar{c}}$ by cross-validating $M_{c_j}$'s for $q_j$'s.
  \item Let $s=|U|$, draw $\omega_i \stackrel{iid}{\sim} \calN(0,\sigma^{-2}I_s)$, $b_i \stackrel{iid}{\sim} \Unif[0,2\pi]$ for $i\in\{1,\ldots,D\}$; keep the set $\{(\omega_i,b_i)\}_{i=1}^D$ fixed henceforth.
  \item Let $\{\alpha_1,\ldots,\alpha_r\}=V$. Generate the data-set of random kitchen sink features, output projection coefficient vector pairs $\{(z(\avec_U(\hP_i)),\avec_V(\hQ_i))\}_{i=1}^N$. Let $\hPsi = (\bZ^T\bZ)^{-1}\bZ^T\bY_{1:r} \in\R^{D\times r} $ where $\bZ=[z(\avec_U(\hP_1))\cdots z(\avec_U(\hP_N)) ]^T \in\R^{N \times D} $, $\bY_{1:r}=[\vY_{\alpha_1},\ldots,\vY_{\alpha_r}]\in\R^{N\times r}$. Note that $\bZ^T\bY_{1:r}$ and $\bZ^T\bZ$ can be computed efficiently using parallelism.
  \item For all future query input functional observations $\hP_0$, estimate the projection coefficients of the corresponding output function as $\hat{f}_{1:r}(\vp_0) = \hPsi^T z(\avec_U(\vp_0))$.
\end{enumerate}

\section{Theory}
\label{theory}
We analyze the $L_2$ risk for the 3BE estimator below. We assume that input/output functions belong to a Sobolev Ellipsoid function class and that the mapping between input and output functions is in $\calF_\sigma$ \eqref{eq:func_class_set}.

\subsection{Assumptions}
\subsubsection{Sobolev Ellipsoid Function Classes}
We shall make a Sobolev ellipsoid assumption for classes $\calI$ and $\calO$.
Let $a(h)\equiv \{a_\alpha(h)\}_{\alpha\in\Z^d}$. Suppose that the projection
coefficients $a(p) = \{a_\alpha(p)\}_{\alpha\in \Z^l}$ and $a(q) =
\{a_\alpha(q)\}_{\alpha\in \Z^k}$ are as follows for $p\in\calI$, $q\in\calO$:
\begin{align}
\calI &= \{ p : a(p) \in \Theta_l(\nu_\subI,\gamma_\subI,\bA_\subI),\norm{p}_\infty \leq \bA_\subI \} \label{eq:sob-ellp-inp2}\\
\calO &= \{ q : a(q) \in \Theta_k(\nu_\subO,\gamma_\subO,\bA_\subO),\norm{q}_\infty \leq \bA_\subO \} \label{eq:sob-ellp-inp}
\end{align}
where $\nu_\subI,\gamma_\subI\in \R_{++}^l$, $\nu_\subO,\gamma_\subO\in \R_{++}^k$, $A_\subI,A_\subO\in \R_{++}$, $\R_{++}=(0,\infty)$, and
\begin{align}
    \Theta_d(\nu,\gamma,\bA) &= \Big\{\{a_\alpha\}_{\alpha\in \Z^d} : \sum_{\alpha\in\Z^d} a_\alpha^2 \kappa_\alpha^2(\nu,\gamma) < \bA \Big\} \\
\kappa_\alpha^2(\nu,\gamma) &= \sum_{i=1}^d(\nu_i|\alpha_i|)^{2\gamma_i}\ \mathrm{for}\ 
\nu_i,\gamma_i,\bA > 0.
\end{align}
See \cite{ingster2011estimation,laurent1996efficient} for other work using
similar Sobolev elipsoid assumptions. The assumption in
\eqref{eq:sob-ellp-inp} will control the
tail-behavior of projection coefficients and allow one to effectively estimate
$p\in \calI$ and $q\in \calO$ using a finite number of projection coefficients
on the empirical functional observation.

Suppose as before that function
$h$ is such that $a(h)\in\Theta_d(\nu,\gamma,\bA)$ has a corresponding set of evaluations $H =
\{y_j = h(u_j)+ \epsilon_j\}_{j=1}^r$ where $u_j\iid \Unif([0,1]^d)$ and
$\EE{\epsilon_j}=0$, $\EE{\epsilon_j^2}<\infty$. Then, $\tilde{h}$, the estimate of $h$, is:
\begin{align}
\tilde{h}(x) &= \sum_{\alpha\ :\ \kappa_\alpha(\nu,\gamma)\leq t}a_\alpha(H)\varphi_\alpha(x) \quad \where \\
a_\alpha(H) &= \frac{1}{r}\sum_{j=1}^{r} y_j\varphi_\alpha(u_j) \label{eq:coef-est-th}.
\end{align}
Choosing $t$ optimally\footnote{See appendix for details.} can be shown to lead
to $\E[\norm{\tilde{h}-h}_2^2]=O(r^{-\frac{2}{2+\gamma^{-1}}})$, where
$\gamma^{-1}=\sum_{j=1}^{d}\gamma_j^{-1}$, $r \rightarrow \infty$. Thus, we can represent $h$ using a finite number of projection coefficients $\avec_t(H)=(a_\alpha(H) :\kappa_\alpha(\nu,\gamma)\leq t)^T$; this allows one to approximate the FFR problem as a regression problem over finite vectors $\avec_t(p)$ and $\avec_{t'}(q)$. Note that our choice of sets $M_t$ \eqref{eq:cv_Mt} in $\S$\ref{sec:pccv} corresponds to the estimator in \eqref{eq:coef-est-th} with $\nu,\gamma=\vec{1}$. Varying $t$ in this case will still be adaptive to the smoothness of $h$, and the number of points in $H$.

\subsubsection{Function to Funcion Mapping}
Recall that we take output functions to be:
\begin{align*}
q(x) = [f(p)](x) &= \sum_{\alpha\in\Z^k} f_\alpha(p)\varphi_\alpha(x)
\end{align*}
where $f_\alpha(p) = a_\alpha(f(p))$.
An our assumption of the class of mappings is:
\begin{align*}
\calF_\sigma &= \{f\ :\ \forall \alpha\in\Z^k\ \norm{\theta_{\alpha}}_1\leq B_{\alpha}, f_{\alpha} \text{ as in } \eqref{eq:q_form}  \}
\end{align*}
Suppose further that:
\begin{align}
\sum_{\alpha\in\Z^k} B_{\alpha}^2 \kappa_\alpha^2(\nu_\subO,\gamma_\subO)\leq \bA_\subO 
\label{eq:func_class}.
\end{align}
Hence, if $f \in \calF_\sigma$ then $q = f(p) \implies q\in\calO$ since $|f_\alpha(p)|\leq \norm{\theta_{\alpha}}_1\leq B_{\alpha}$ and \eqref{eq:func_class} holds.

\subsection{Risk Upperbound}
Below we state our main theorem, upperbounding the risk of the 3BE.
\begin{theorem}
Let a small constant $\delta>0$ be fixed. Suppose that $\hat{q}_0(x) = \sum_{\alpha \in M^\subO_u} \hat{f}_\alpha(\vp_0)\varphi_\alpha(x)$, $\hat{f}_\alpha(\vp_0)$ given by \eqref{eq:OLSest}. Furthermore, suppose that \eqref{eq:sob-ellp-inp} holds and $f \in \calF_\sigma$ as in \eqref{eq:func_class}. Moreover, assume that \eqref{eq:func_obs} holds and $n_i,m_i \asymp n$. Also, assume that the number of RKS features $D$ \eqref{eq:rks_feats} is taken to be $D\asymp n\log(n)$.  Then,
\begin{align}
&\EE{\norm{q_0-\hat{q}_0}^2_2} \\
&\leq O\left( \left(n^{-1/(2+\gamma_\subI^{-1})}+\frac{n\log(n)\log(N)}{N}\right)^{2/(2+\gamma_\subO^{-1})}\right) \label{eq:f2f_rate}\\
& \text{with probability at least }1-\delta. \nonumber
\end{align}
\end{theorem}
See appendix for proof. The rate \eqref{eq:f2f_rate} yields consistency for our estimator if $n\log(n) = o(N/\log(N))$; that is, so long as one is in the large data-set domain where the number of instances is larger than the number of points in function observations. Note that the first summand in \eqref{eq:f2f_rate} is similar to typical functional estimation rates, and it stems from our approximation with bases; the second summand is akin to a linear regression rate, and it stems from our OLS estimation \eqref{eq:OLSest}.


\section{Experiments}
\label{experiments}
Below we show the improvement of the 3BE over previous FFR approaches in several real-world data-sets. Empirically, the 3BE proves to be the most general, quickest, and effective estimator. Unlike previous time-series FFR approaches, the 3BE easily lends itself to working over distributions. Moreover, unlike previous nonparametric FFR estimators the 3BE does not need to compute pairwise kernel evaluations, making it much more scalable. All differences in MSE were statistically significant ($p < 0.05$) using paired t-tests.

\subsection{Rectifying 2LPT Simulations}
Numerical simulations have become an essential tool to study cosmological structure formation. 
Astrophysics use N-body simulations \cite{trac2006out} to study the gravitational evolution of collisionless particles like dark matter particles. Unfortunately, N-body simulations require forces among particles to be recomputed over multiple time intervals, leading to a large magnitude of time steps to complete a single simulation. In order to mitigate the large computational costs of running N-body simulations, often simulations based on Second Order Lagrange Perturbation Theory (2LPT) \cite{scoccimarro1998transients} are used. 
\begin{figure}
\label{fig:part}
  \centering
  \includegraphics[width=.45\textwidth]{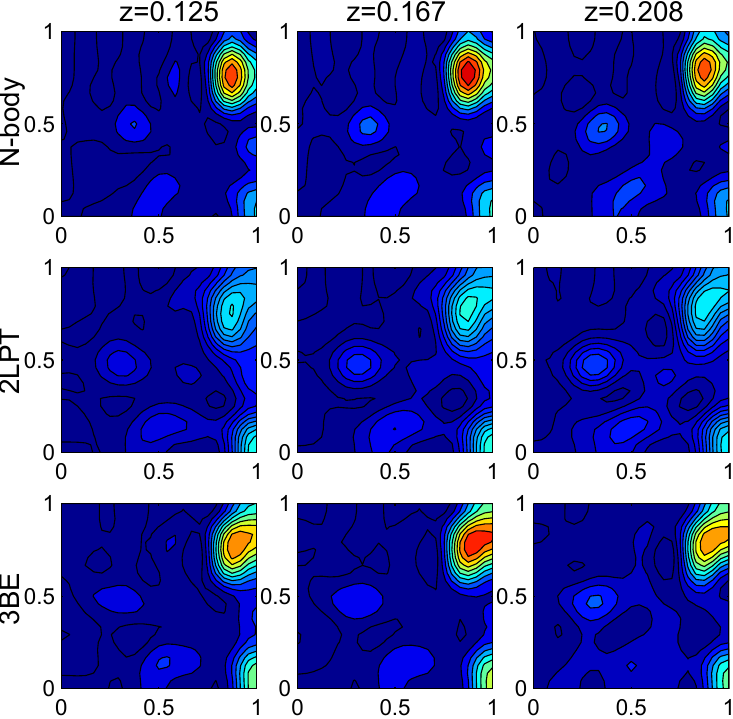}
    \caption{Slices of particle pdfs.}
\end{figure}
Although 2LPT simulations are several orders of magnitude faster, they prove to be inaccurate, especially at smaller scales. In this experiment we bridge the gap between the speed of 2LPT simulations and the accuracy of N-body simulations using FFR and the 3BE. Namely, we regress the mapping between a distribution of particles in an area coming from a 2LPT simulation and the distribution of the particles in the same area under an equivalent N-body simulation.

\begin{wrapfigure}{r}{0.3\textwidth}
\small
\centering
\begin{tabular}{l{c}{c}}
Method         	& MSE  		& MPT\\
\hline
\textbf{3BE}    & \textbf{4.958}  	& \textbf{0.009}\\
LSE            	& 6.816		& 4.977\\
2LPT           	& 6.424  	& NA\\
AD				& 9.289		& NA
\end{tabular}
\captionof{table}{\small \label{tbl:phys_res}MSE and MPT(s) results.}\label{tbl:astro}
\end{wrapfigure}
We regress the distribution of 3d (spatial) N-body simulation particles in $16\ \mathrm{Mpc}^3$ cubes when given the distribution of particles of the 2LPT simulation in the same cube (note that each distribution is estimated through the set of particles in each cube). A training-set of over 900K pairs of 2LPT cube sample-set/N-body cube sample-set instance was used, along with a test-set of 5K pairs. The number of projection coefficients used to represent input and output distributions was 365/401 respectively, chosen by cross-validating the density estimates. We chose the number of RKS features to be 15K based on rules-of-thumb. We cross-validated the $\sigma$ and $\lambda$ parameters of the ridge variant 3BE \eqref{eq:ridgeest} and the smoothing parameter of the LSE and reported back the MSE and mean prediction time (MPT, in seconds) of our FFR estimates to the distributions truly coming directly through N-body simulation (Table \ref{tbl:astro}); we also report the MSE of predicting the average output distribution (AD).

We see that the 3BE is about $500\times$ faster than the LSE in terms of prediction time and achieved an improvement in $R^2$ of over $50\%$ over using the distribution coming directly from the 2LPT simulation (2LPT). Note also that the LSE does not achieve an improvement in MSE over 2LPT.


\subsection{Time-series Data}
We compared the performance of the 3BE in time-series prediction problems to using the LSE and widely used time-series prediction methods like
Dynamics Mining with Missing values (DynaMMo) \cite{li2009dynammo} and Kernel Embedded HMMs (SHMM) \cite{song2010hilbert}. 
DynaMMo is a latent-variable
probabilistic model trained with EM aimed at predicting
data that is missing in chunks and not just in a single time-step (as we also
attempt with our functional responses).

\subsubsection{Forward Prediction with Music Data}
\begin{figure}
\centering
\includegraphics[width=.45\textwidth]{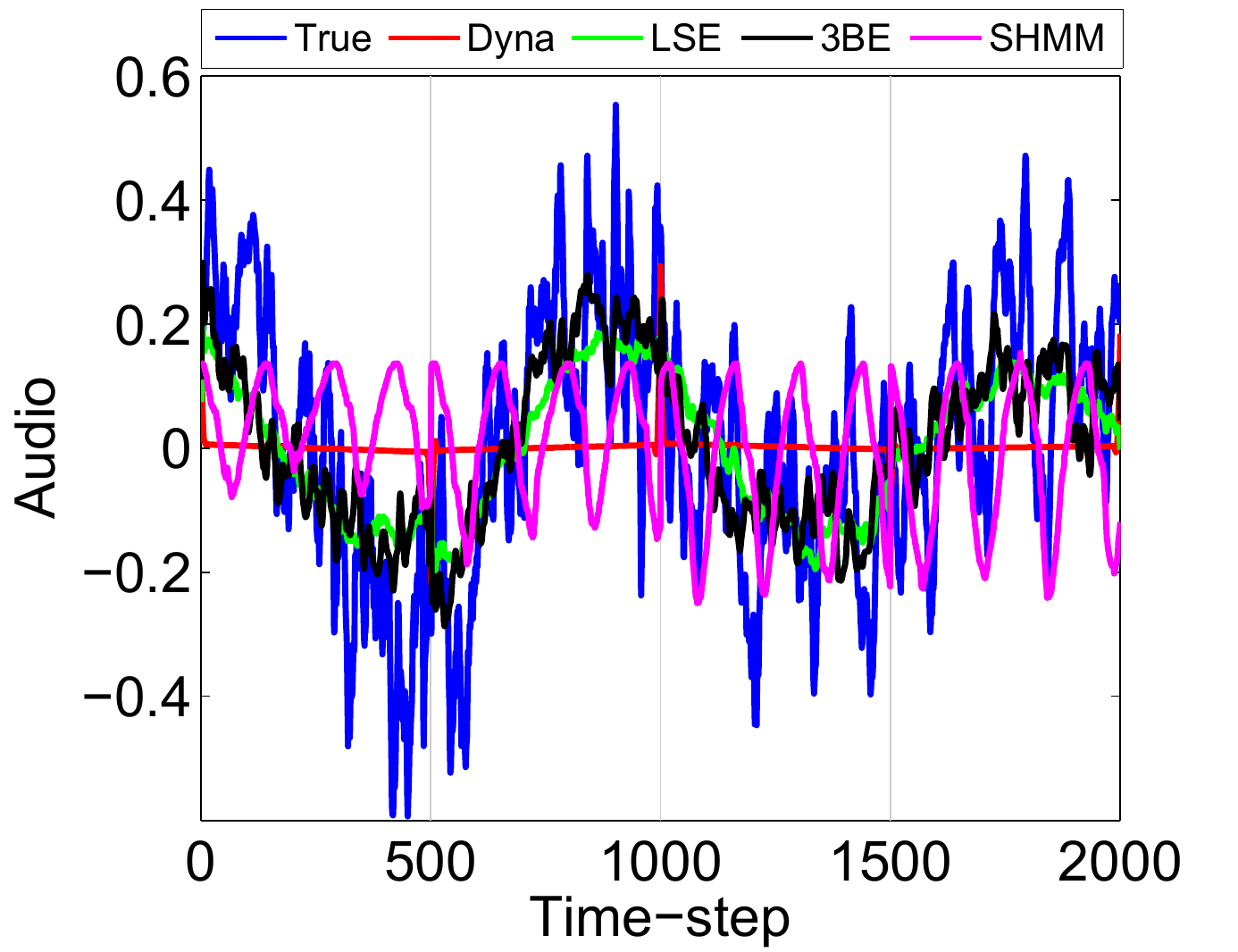}
\caption{\label{fig:mus_res}Example audio predictions; segments separated with vertical lines.}
\end{figure}
Music data presents a particularly interesting application of forward
prediction for time-series. That is, given a short segment of audio data from a
piece of music, can we predict the audio data in the short segment that
follows? Uses for forward prediction with music include compression and music similarity.

In this experiment, we use a 30 second clip, sampled at 44.1 kHz from the song
``I Turn To You'' by the artist Melanie C. We extract a mono signal of the sound clip and use the first $85\%$ for
training and hold-out, and the final $15\%$ for testing. To perform forward
prediction in the test set, we take a $500$ time-step segment of the (true)
music time-series as input and use it to predict the following $500$
time-steps. We repeat this sequentially over consecutive disjoint segments in
the test set until we have made predictions for the entire test set. In total our data-set consisted of about $2,200$ training instances. For audio
prediction with the 3BE we use the ridge variant \eqref{eq:ridgeest}. We use
150 trigonometric basis functions for both input and output functions, and 5000
RKS basis functions (both quantities chosen via rules of thumb). We then
cross-validate the bandwidth and $\lambda$ penalty parameters.

\begin{wrapfigure}{r}{0.2\textwidth}
\small
\centering
\begin{tabular}{l{c}}
Method         & MSE  \\
\hline
3BE 		    & 0.0327   \\
LSE            & 0.0351  \\
Dyna           & 0.0492  \\
SHMM			& 0.1082
\end{tabular}
\captionof{table}{\small \label{tbl:mus_res}Audio MSE.}
\end{wrapfigure}
We cross-validated the number of dimensions for hidden-states for DynaMMo, and
the bandwidth parameter for the LSE. The mean squared error (MSE) on the
test-set is reported in Table \ref{tbl:mus_res} for each method.
The 3BE achieves the lowest estimation error. Furthermore, looking at Figure
\ref{fig:mus_res} it is apparent that the 3BE outperforms the other methods in
terms of capturing the structure of the audio data. The quality of the audio
predicted with the 3BE is also superior to the other methods (hear predicted sound clips in supplemental
materials). 
Furthermore, DynaMMo
takes over 4 hours to learn a model given a fixed hidden state dimensionality
with no missing data (and even longer if also predicting missing data), where
as the 3BE takes only about 2 minutes to cross-validate and perform predictions
(a speed-up of over $7000\times$). Similarly the 3BE was over $5000\times$ faster than SHMM for predictions. Additionally, even though the data-set is of a smaller scale, the 3BE still enjoys a $3\times$ speedup over LSE for prediction time.


\subsubsection{Co-occurring Predictions with Joint Motion Capture Data}
\begin{figure}
  \centering
  \includegraphics[width=.45\textwidth]{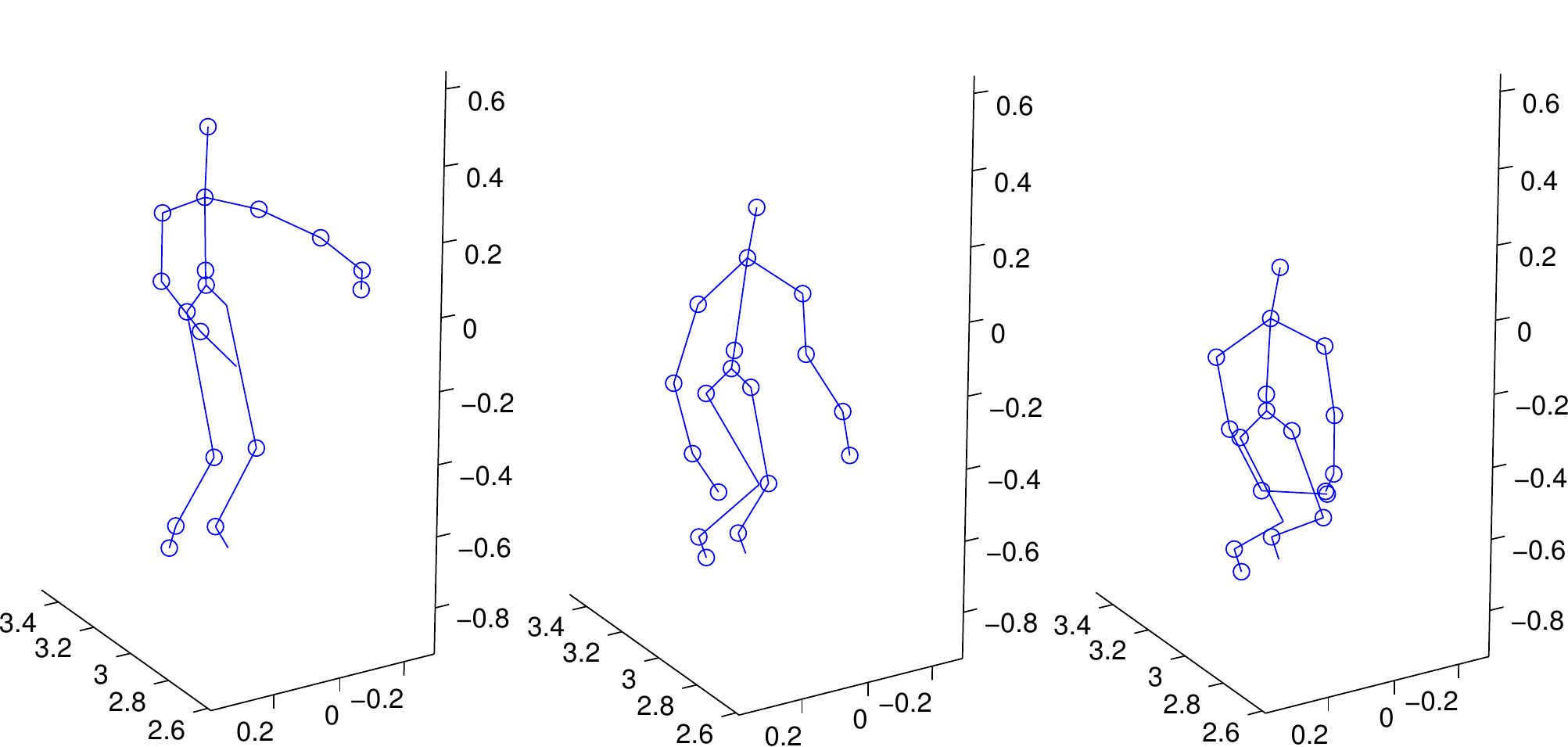}
\caption{\label{fig:keyframes} Example ``duck'' frames.}
\end{figure}
Next, we explore predicting co-occurring time-series with motion capture
(MoCap) data. 
We use the MSRC-12 Data-set
\cite{fothergill2012instructing}.
The 3d positions are provided for 20 total joints. 
We look to predict the time-series of the
position of an unobserved joint over a $T$ time-step segment given 
time-series data (one function for each joint's x, y, or z position) for $R$
observed joints for the segment.
\begin{figure}
  \centering
  \includegraphics[width=.3\textwidth]{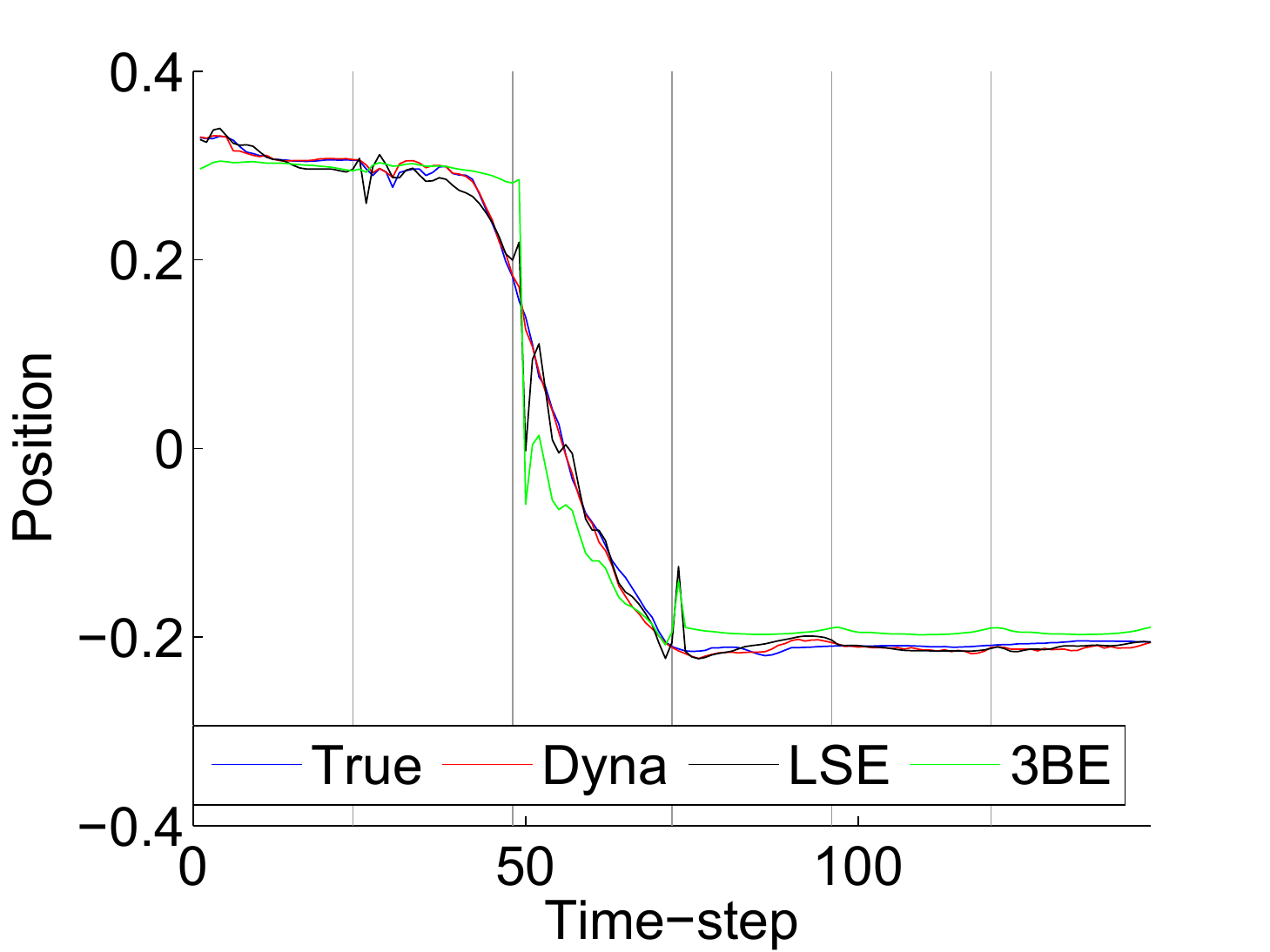}
  \caption{\label{fig:mocap_res}Occluded joint predictions.}
\end{figure}

We performed co-occurring time-series prediction with MoCap data of a subject
performing the gesture ``duck'' (Figure \ref{fig:keyframes}). We randomly chose
10 joints to designate as occluded, and used the other 10 as our non-occluded
joints. We then solved 30 separate FFR problems, where each of the problems had
one of the missing joints' time-series as the output response function (e.g.
missing joint 1's y position or missing joint 4's x position). In each of the
problems, the 30 functions corresponding to the time-series for non-occluded
joint spatial positions were used as inputs (by concatenating the projection coefficients of each input function)
. We considered segments of 24
time-steps for time-series functions. In total we used a training set of about 1100 instances. The number of projection coefficients for
functions was taken to be 10 while the number of RKS features was 250. The same
parameters for all estimators were cross validated as before.

\begin{wrapfigure}{r}{0.2\textwidth}
\small
\centering
\begin{tabular}{l{c}}
Method         & MSE  \\
\hline
3BE 		    & 7.78E-4   \\
LSE            & 1.3E-3  \\
Dyna           & 2.40E-4  \\
\end{tabular}
\captionof{table}{\label{tbl:mocap_res}MoCap MSE.}
\end{wrapfigure}

DynaMMo performs the best (Table \ref{tbl:mocap_res}), which is perhaps not
surprising given that MoCap occlusion prediction was a point of emphasis for
DynaMMo. However, the differences in prediction qualities among the different
methods is not as pronounced in this data-set (Figure \ref{fig:mocap_res}). We
again see a speed up of over 1000x using 3BE over DynaMMo, also there was a speed up of over $30\times$ in prediction time over LSE.

\section{Conclusion}
\label{conclusion}
In conclusion, this paper presents a new estimator, the Triple Basis Estimator
(3BE), for performing function to function regression in a scalable manner.
Since functional data is complex, it is important to have an estimator that is
capable of using massive data-sets in order to achieve a low estimation risk.
To the best of our knowledge, the 3BE is the first nonparametric FFR estimator that is capable to scaling to big data-sets. The 3BE achieves this through the use of a basis representation of input and
output functions and random kitchen sink basis functions. We analyzed the risk
of the 3BE given non-parametric assumptions. Furthermore, we showed 
an improvement of several orders of magnitude for prediction speed and a reduction in error over previous estimators in various real-world data-sets.

\clearpage

\bibliography{main}
\bibliographystyle{plain}


\clearpage
\section*{Appendix}
\subsection*{Multidimensional Projection Series Functional Estimation}
Let $y_{ij}=p_i(u_{ij})+\epsilon_{ij}$. Given $\vp_i$ as in \eqref{eq:func_obs} our estimator for $p_i\in \calI$ \eqref{eq:sob-ellp-inp} will be:
\begin{align}
\tp_i(x) =& \sum_{\alpha\ :\ \kappa_\alpha(\nu,\gamma)\leq t}a_\alpha(\vp_i)\varphi_\alpha(x) \quad \where \\
a_\alpha(\vp_i) =& \frac{1}{n_i}\sum_{j=1}^{n_i} y_{ij}\varphi_\alpha(u_{ij}) .
\end{align}
For readability let
$(\nu,\gamma,\bA)=(\nu_\subI,\gamma_\subI,\bA_\subI)$. First, note that:
\begin{align}
&\EE{\norm{p_i-\tp_i}_2^2}\\
=& \EE{\Norm{ \sum_{\alpha\in\Z} a_{\alpha}(p_i)\varphi_{\alpha} -\sum_{\alpha\in M^\subI_t} a_{\alpha}(\vp_i) \varphi_{\alpha}}_2^2} \nonumber\\
=& \E\Bigg[\int_{\Lambda^l} \Bigg(\sum_{\alpha\in M^\subI_t} (a_{\alpha}(p_i)-a_{\alpha}(\vp_i))\varphi_{\alpha}(x) \nonumber \\ &+\sum_{\alpha\in (M^\subI_t)^c} a_{\alpha}(p_i) \varphi_{\alpha}(x)\Bigg)^2 \ud x\Bigg] \nonumber \\
=& \E \Bigg[\int_{\Lambda^l} \sum_{\alpha\in M^\subI_t} \sum_{\rho\in M^\subI_t} (a_{\alpha}(p_i)-a_{\alpha}(\vp_i)) \nonumber \\
&\qquad\quad (a_{\rho}(p_i)-a_{\rho}(\vp_i)) \varphi_{\alpha}(x)\varphi_{\rho}(x) \ud x\Bigg] \nonumber\\
&+ 2\E \Bigg[\int_{\Lambda^l} \sum_{\alpha\in M^\subI_t} \sum_{\rho\in (M^\subI_t)^c} (a_{\alpha}(p_i)-a_{\alpha}(\vp_i)) \nonumber \\
&\qquad \qquad a_{\rho}(p_i) \varphi_{\alpha}(x)\varphi_{\rho}(x) \ud x\Bigg]\nonumber\\
&+ \E \Bigg[ \int_{\Lambda^l} \sum_{\alpha\in (M^\subI_t)^c} \sum_{\rho\in (M^\subI_t)^c} \nonumber\\ 
&\qquad \qquad a_{\alpha}(p_i)a_{\rho}(p_i)\varphi_{\alpha}(x)\varphi_{\rho}(x) \ud x \Bigg] \nonumber \\
=& \EE{\sum_{\alpha\in M^\subI_t} (a_{\alpha}(p_i)-a_{\alpha}(\vp_i))^2 }+\EE{\sum_{\alpha\in (M^\subI_t)^c} a_{\alpha}^2(p_i) } \label{eq:derisk},
\end{align}
where the last line follows from the orthonormality of $\{\varphi \}_{\alpha\in\Z}$.
Furthermore, note that $\forall p_i\in \calI$:
\begin{align}
\sum_{\alpha\in (M^\subI_t)^c} a_{\alpha}^2(p_i) =& \frac{1}{t^2} \sum_{\alpha\in (M^\subI_t)^c} t^2a_{\alpha}^2(p_i) \\
\leq& \frac{1}{t^2} \sum_{\alpha\in\Z} \kappa^2_\alpha(\nu,\gamma)a_{\alpha}^2(p_i) 
\leq \frac{\bA}{t^2} \label{eq:turn_err}.
\end{align}
Also,
\begin{align*}
\EE{(a_{\alpha}(p_i)-a_{\alpha}(\vp_i))^2 } =& \left(\EE{a_{\alpha}(\vp_i)}-a_{\alpha}(p_i)\right)^2 \\
&+ \Var\left[a_{\alpha}(\vp_i)\right].
\end{align*}
We may see that $a_{\alpha}(\vp_i)$ is unbiased. Let $u\sim U([0,1]^l)$, then: 
\begin{align*}
&a_{\alpha}(p_i) \\
&= \int_{\Lambda^l} \varphi_\alpha(z)p_i(z)(1)\ud z = \EE{\varphi_\alpha(u)p_i(u)}\\
&= \EE{\frac{1}{n_i}\sum_{j=1}^{n_i} p(u_{ij})\varphi_\alpha(u_{ij})} \\
&= \EE{\frac{1}{n}\sum_{j=1}^{n} p(u_{ij})\varphi_\alpha(u_{ij})} + \EE{\frac{1}{n}\sum_{j=1}^{n} \epsilon_{ij} \varphi_\alpha(u_{ij})} \\
&= \EE{\frac{1}{n_i}\sum_{j=1}^{n_i} y_{ij} \varphi_\alpha(u_{ij})} = \EE{a_{\alpha}(\vp_i)}
\end{align*}
Also,
\begin{align*}
\Var\left[a_{\alpha}(\vp_i)\right] =& \frac{1}{n_i^2}\sum_{j=1}^{n_i} \Var\left[y_{ij}\varphi_\alpha(u_{ij})\right]\\
\leq& \frac{1}{n_i^2}\sum_{j=1}^{n_i} \EE{(y_{ij}\varphi_\alpha(u_{ij}))^2}\\
\leq& \frac{\mxphi}{n_i^2}\sum_{j=1}^{n_i} \EE{(p_i(u_{ij})+\epsilon_{ij})^2}\\
=& \frac{\mxphi}{n_i^2}\sum_{j=1}^{n_i} \EE{p_i(u_{ij})^2+2p_i(u_{ij})\epsilon_{ij}+\epsilon_{ij}^2}\\
\leq& \frac{\mxphi (\bA^2+\varsigma^2) n_i}{n_i^2}
=O(n_i^{-1}),
\end{align*}
where $\mxphi \equiv \max_{\alpha\in\Z^l}\norm{\varphi_\alpha}_\infty$. Thus,
\begin{align*}
\EE{\norm{p_i-\tp_i}_2^2} \leq \frac{C_1|M^\subI_t|}{n_i} + \frac{C_2}{t^2}.
\end{align*}
First note that if we have a bound $\forall \alpha\in M^\subI_t,\ |\alpha_i|\leq c_i$ then $|M^\subI_t|\leq \prod_{i=1}^l(2c_i+1)$, by a simple counting argument. Let $\lambda=\mathrm{argmin}_i\nu_i^{2\gamma_i}$. For $\alpha\in M^\subI_t$ we have: 
\begin{align}
\sum_{i=1}^l|\alpha_i|^{2\gamma_i} \leq \frac{1}{\nu_\lambda^{2\gamma_\lambda}}\sum_{i=1}^l(\nu_i|\alpha_i|)^{2\gamma_i} = \frac{\kappa^2_\alpha(\nu,\gamma)}{\nu_\lambda^{2\gamma_\lambda}} \leq \frac{t^2}{\nu_\lambda^{2\gamma_\lambda}}, \label{eq:nterm_bnd1}
\end{align} 
and
\begin{align}
|\alpha_i|^{2\gamma_i} \leq \sum_{i=1}^l|\alpha_i|^{2\gamma_i} \leq {t^2}{\nu_\lambda^{-2\gamma_\lambda}} \implies |\alpha_i| \leq \nu_\lambda^{-\frac{\gamma_\lambda}{\gamma_i}}t^{\frac{1}{\gamma_i}}. \label{eq:nterm_bnd2}
\end{align}
Thus, $|M^\subI_t|\leq \prod_{i=1}^{l}(2\nu_\lambda^{-\frac{\gamma_\lambda}{\gamma_i}}t^{\frac{1}{\gamma_i}}+1)$. Thus, $|M^\subI_t|= O\left( t^{\gamma^{-1}} \right)$ where $\gamma^{-1}=\sum_{j=1}^{l}\gamma_j^{-1}$. 
Hence,
\begin{align*}
&\frac{\partial}{\partial t} \left[\frac{C_1t^{\gamma^{-1}}}{n_i} + \frac{C_2}{t^2}\right] = \frac{C_1't^{\gamma^{-1}-1}}{n_i} - C_2't^{-3}= 0 \implies \\
&t = C n^{\frac{1}{2+\gamma^{-1}}} \implies\\
&\EE{\norm{p_i-\tp_i}_2^2} \leq \frac{C_1|M^\subI_t|}{n_i} + \frac{C_2}{t^2} =O\left(n_i^{-\frac{2}{2+\gamma^{-1}}}\right). 
\end{align*}
A similar result may be reached for $\tq_i$ functions.

\section*{Theory}

\subsection*{Assumptions}
\begin{enumerate}[label=\textbf{A.\arabic*}]
\itemsep0em 
\item{ \label{asmp:sob}
{\em Sobolev Input/Output Functions.} Suppose that \eqref{eq:sob-ellp-inp2} \eqref{eq:sob-ellp-inp} hold.
}
\item{ \label{asmp:sob}
{\em FFR Mapping.} We shall assume that $f \in \calF_\sigma$ as in \eqref{eq:func_class_set}, \eqref{eq:func_class}.
}
\item{ \label{asmp:sampsize}
{\em Functional Observations } Suppose that \eqref{eq:func_obs} holds and $n_i,m_i \asymp n$. 
}
\item{ \label{asmp:rks}
{\em RKS Features } Suppose that the number of RKS features $D$ \eqref{eq:rks_feats} is taken to be $D\asymp n\log(n)$. 
}
\end{enumerate}

\begin{lemma}
\label{lm:q_bnd}
Let $\hat{q}_0(x) = \sum_{\alpha \in M^\subO_u} \hat{f}_\alpha(\vp_0)\varphi_\alpha(x)$. If $\forall \alpha\in\Z^k$, $\EE{(a_{\alpha}(q_0)-\hat{f}_\alpha(\vp_0))^2}=O(\calR(N,n))$, then $\EE{\norm{q_0-\hq_0}_2^2} = O\left(\calR(N,n)^{2/2+\gamma_\subO^{-1}}\right)$, where $\gamma_\subO^{-1}=\sum_{j=1}^{k}(\gamma_\subO)_j^{-1}$.
\end{lemma}
\begin{proof}
Let $M^\subO_u$ be defined as in \eqref{eq:Mt}, note that: 
\begin{align}
\norm{q_0-\hq_0}_2^2 =& \sum_{\alpha\in M^\subO_u} (a_{\alpha}(q_0)-\hat{f}_\alpha(\vp_0))^2 \\
&+\sum_{\alpha\in (M^\subO_u)^c} a_{\alpha}^2(p_i)  \label{eq:pred_risk},
\end{align}
by the orthonormality of $\{\varphi_\alpha \}_{\alpha\in\Z}$ (see above).
Then,
\begin{align}
\EE{\norm{q_0-\hq_0}_2^2} \leq& \sum_{\alpha\in M^\subO_u} \EE{(a_{\alpha}(q_0)-\hat{f}_\alpha(\vp_0))^2} \\
&+\EE{\sum_{\alpha\in (M^\subO_u)^c} a_{\alpha}^2(q_0)}.
\end{align}
Furthermore, since $q_0\in\Theta_k(\nu_\subO,\gamma_\subO,\bA_\subO)$,
\begin{align}
\sum_{\alpha\in (M^\subO_u)^c} a_{\alpha}^2(q_0) =& \frac{1}{u^2} \sum_{\alpha\in(M^\subO_u)^c} u^2a_{\alpha}^2(q_0) \\
&\leq \frac{1}{u^2} \sum_{\alpha\in\Z} \kappa^2_\alpha(\nu_\subO,\gamma_\subO)a_{\alpha}^2(q_0) \leq \frac{\bA_\subO^2}{u^2} \label{eq:turnc_bnd}.
\end{align}
Thus,
$
\EE{\norm{q_0-\hq_0}_2^2} = O\left(\calR(N,n)|M^\subO_u|+\frac{\bA_\subO^2}{u^2}\right).
$

For simplicity of notation let
$(\nu,\gamma,\bA)=(\nu_\subO,\gamma_\subO,\bA_\subO)$. By an argument similar to 
\eqref{eq:nterm_bnd1} and \eqref{eq:nterm_bnd2} we have that $|M^\subO_u|= O\left( u^{\gamma^{-1}} \right)$ where $\gamma^{-1}=\sum_{j=1}^{l}\gamma_j^{-1}$.
Hence choosing $u\asymp
\calR(N,n)^{-1/(2+\gamma^{-1})}$ yields
$
\EE{\norm{q_0-\hq_0}_2^2} = O\left(\calR(N,n)^{2/2+\gamma_\subO^{-1}}\right).
$
\end{proof}

\begin{lemma}
\label{lm:R_bnd}
Let a small constant $\delta>0$ be fixed. Suppose that $\hat{f}_\alpha(\vp_0)$
is given by \eqref{eq:OLSest}. Then, asymptotically $\forall \alpha\in\Z^k$,
\begin{align*}
&\EE{(f_\alpha(p)-\hat{f}_\alpha(\vp_0))^2} = \\
& O\left(n^{-1/(2+\gamma_\subI^{-1})}+\max(\sfrac{1}{n},B_\alpha)\frac{n\log(n)\log(N)}{N}\right)\\
& \text{with probability at least }1-\delta.
\end{align*}
\end{lemma}
\begin{proof}
Note that $\hat{f}_\alpha(\vp_0)$ is a function to real estimator attempting to
estimate the mapping $p_0\mapsto f_\alpha(p_0)$. Note that
$a_\alpha(q_i)=a_\alpha(\vq_i)+\epsilon_{\alpha i}$, where
$\EE{\epsilon_{\alpha}}=0$ and $\Var[\epsilon_\alpha]=O(\sfrac{1}{n})$ (see
above). Also,  $\hat{f}_\alpha(\vp_0)$ is trained using a data-set $\calD_\alpha =
\{(\vp_i,a_\alpha(\vq_i))\}_{i=1}^N = \{(\vp_i,a_\alpha(q_i)+\epsilon_{\alpha
i})\}_{i=1}^N$. Thus, a straightforward analogue (using general functions
rather than distributions) of the rate derived in \cite{oliva2013fast} yields
the result.
\end{proof}

\begin{theorem}
Let a small constant $\delta>0$ be fixed. Suppose that $\hat{q}_0(x) = \sum_{\alpha \in M^\subO_u} \hat{f}_\alpha(\vp_0)\varphi_\alpha(x)$, $\hat{f}_\alpha(\vp_0)$ given by \eqref{eq:OLSest}. Furthermore, suppose that \eqref{eq:sob-ellp-inp} holds and $f \in \calF_\sigma$ as in \eqref{eq:func_class}. Moreover, assume that \eqref{eq:func_obs} holds and $n_i,m_i \asymp n$. Also, assume that the number of RKS features $D$ \eqref{eq:rks_feats} is taken to be $D\asymp n\log(n)$.  Then,
\begin{align*}
&\EE{\norm{q_0-\hat{q}_0}^2_2}\\
&\leq O\left( \left(n^{-1/(2+\gamma_\subI^{-1})}+\frac{n\log(n)\log(N)}{N}\right)^{2/(2+\gamma_\subO^{-1})}\right)\\
& \text{with probability at least }1-\delta.
\end{align*}
\begin{proof}

Follows from Lemmas \ref{lm:q_bnd} and \ref{lm:R_bnd}.
\end{proof}
\end{theorem}

\end{document}